%% file: main.tex
\pdfoutput=1
\documentclass[letterpaper, 10 pt, conference]{ieeeconf}
\usepackage{type1cm}
\IEEEoverridecommandlockouts                              
\usepackage{graphicx} 
\usepackage{wrapfig}
\usepackage{amsfonts}
\usepackage{amsmath}
\usepackage{amssymb}
\usepackage{float}
\usepackage{adjustbox}

\usepackage{tikz}
\usetikzlibrary{positioning, shapes, arrows, calc, decorations.pathreplacing, arrows.meta, backgrounds, patterns, overlay-beamer-styles, shapes.symbols, decorations.pathmorphing}
\usepackage{tikz-cd}
\usepackage[dvipsnames]{xcolor}

\input{rov}
\newcommand\copyrighttext{%
  \footnotesize This work has been submitted to the IEEE for possible publication. Copyright may be transferred without notice, after which this version may no longer be accessible.}
\newcommand\copyrightnotice{%
\begin{tikzpicture}[remember picture,overlay]
\node[anchor=south,yshift=10pt] at (current page.south) {\fbox{\parbox{\dimexpr\textwidth-\fboxsep-\fboxrule\relax}{\copyrighttext}}};
\end{tikzpicture}%
}

\newtheorem{theorem}{Theorem}[section]
\newtheorem{lemma}{Lemma}[section]
\newtheorem{remark}{Remark}[section]
\title{Complex-Valued GNNs for Distributed Basis-Invariant Control of Planar Systems}
\author{Samuel Honor, Mohamed Abdelnaby and Kevin Leahy
\thanks{S. Honor, M. Abdelnaby and K. Leahy are with the Robotics Engineering Department, Worcester Polytechnic Institute, Worcester, MA 01609, USA. 
        {\tt\small \{sjhonor,mabdelnaby,kleahy\}@wpi.edu}.   
        This work was supported by ONR under grant number N00014-25-1-2225.}%
}

\date{February 2026}
\tikzset{
    fr/.pic={
        \draw[red, thick] (0,0) -- (0.6, 0);
        \draw[green, thick] (0,0) -- (0,0.6);
        \filldraw[black] (0,0) circle (0.05);
        \coordinate (-mid) at (0.5, 0.5) {};
        \coordinate (-o) at (0,0);
    }
}

\tikzset{
    frov/.pic={
    \pic[transform shape] (r) {rov};
    \pic[rotate=90, transform shape] (f) at (r-center) {fr};
    \coordinate (-center) at (r-center);
    }
}

\tikzset{link/.style={
    thick, gray, opacity=0.6
}}

\tikzset{
  enc/.style={magenta, line width=0.7mm, arrows={-Latex[angle=40:2.5mm, width=2mm, length=2mm]}}
}

\tikzset{latent/.style={enc, decorate, decoration=snake}}

\begin{document}

\maketitle
\copyrightnotice
\begin{abstract}
    Graph neural networks (GNNs) are a well-regarded tool for learned control of networked dynamical systems due to their ability to be deployed in a distributed manner. However, current distributed GNN architectures assume that all nodes in the network collect geometric observations in compatible bases, which limits the usefulness of such controllers in GPS-denied and compass-denied environments. This paper presents a GNN parametrization that is globally invariant to choice of local basis. 2D geometric features and transformations between bases are expressed in the complex domain. Inside each GNN layer, complex-valued linear layers with phase-equivariant activation functions are used. When viewed from a fixed global frame, all policies learned by this architecture are strictly invariant to choice of local frames. This architecture is shown to increase the data efficiency, tracking performance, and generalization of learned control when compared to a real-valued baseline on an imitation learning flocking task. 
\end{abstract}
\section{Introduction}
As described in \cite{vasile_translational_2018}, controllers for pairwise interaction systems that have a \textit{quasi-linear} form can provide control independent of a global reference frame. Many controllers for "geometric" swarming tasks such as flocking \cite{jadbabaie_coordination_2003} distance-only formation control \cite{CAO2011776}, and---as \cite{vasile_translational_2018} remark---even $N$-body Hamiltonian dynamics \cite{meyer_introduction_2009} have such quasi-linear forms and are therefore invariant to choice of reference frame. However, it is not always practical to develop a closed-form quasi-linear controller for a given task. For abstract tasks, it is often desirable to use a learning approach for multi-robot control.

Graph Neural Networks (GNNs) have been proven both in theory and in practice to be powerful tools for learning geometric multi-robot control tasks. \cite{xu_how_2019} proposes \textit{algorithmic alignment} as a metric for predicting whether a neural network architecture can efficiently learn a class of algorithms. The aggregation over messages from neighbors in message-passing GNNs mirror the summations found in nearest-neighbor control laws, so it is expected that GNNs would be able to learn nearest-neighbor controllers well. Indeed, works such as ~\cite{tolstaya-learning-2020}, and ~\cite{chen-learning-2024} have demonstrated that GNNs can learn nearest-neighbor controllers that are resilient to temporal delays and changing graph support. However, these GNN-based controllers still require a global reference frame.

The need for a global reference frame in contemporary distributed GNN-based controllers arises from the need to aggregate latent space encodings. Even if---as required in \cite{vasile_translational_2018}---each agent can observe their neighbors' raw geometric encodings (such as velocity) in its own frame, the latent-space representation each agent produces in its first layer is expressed in a high-dimensional basis is related to agents' body frame. In the second layer onward, these high-dimensional encodings cannot be observed and must instead be transmitted. However, all the encodings an agent receives in the second layer onward will be expressed in different latent-space frames, so summing them is mathematically nonsensical unless they are first transformed into a common frame. This paper presents a complex-valued parameterization for a distributed GNN that makes this latent-space transformation possible.






\section{Background}

\subsection{Graph Neural Networks}
A layer in a standard message-passing graph neural networks is constructed as follows:
\begin{equation}\label{eq:mpgnn}
    \mathbf{x}^{(l+1)}_i = \phi \left(\mathbf{x}^{(l)}_i, \bigoplus_{j\in\mathcal{N}(i)}\psi\left(\mathbf{x}^{(l)}_i, \mathbf{x}^{(l)}_j, \mathbf{e}_{ij}\right)\right)
\end{equation}
Where $\psi$ and $\phi$ are learnable functions, and $\bigoplus$ is a permutation-invariant aggregation function such as sum or mean \cite{bronstein_geometric_2021}. Like in feedforward neural networks, graph neural networks can be composed by applying one layer after another with a nonlinearity in between. Since each node's latent encoding includes information from its neighborhood, composing multiple GNN layers allows agents to leverage information from far-away neighbors without the need for expensive flood routing. 

Many variants of message-passing GNNs exist \cite{hamilton_inductive_2018, kipf_semi-supervised_2017, xu_how_2019} and have proven useful in multi-agent control problems \cite{tolstaya-learning-2020, chen-learning-2024, zhang_threat-adaptive_2026}. This paper adapts the GraphSAGE SAGEConv layer \cite{hamilton_inductive_2018}, which is formulated as follows:
\begin{equation}
    \mathbf{x}_i^{(l+1)} = W^{(l)}\cdot\begin{bmatrix}
        \mathbf{x}_i^{(l)} \\
        \underset{j\in\mathcal{N}(i)}{\text{mean}}\,\mathbf{x}_j^{(l)}
    \end{bmatrix}
\end{equation}

\subsection{Complex-Value Neural Networks}\label{sec:cvnns}
Complex-valued neural networks (CVNNs) \cite{hammad_comprehensive_2024} have shown promise for controlling systems with ingrained symmetrical properties, such as kinematic chains \cite{maeda_robot_2014}. The structure of a feedforward CVNNs is very similar to that of a real-valued network, just with the inputs, outputs, and weight matrices all expressed in the complex field. 

One notable difference between CVNNs and real-value networks is the activation function (AF). Since scalars in $\mathbb{C}$ have two components, there are many different ways to apply an AF. Split real-imaginary AFs individually apply some real valued nonlinearities to the real and imaginary part of a layer output; and split phase-amplitude AFs apply the nonlinearities independently to the phase and amplitude \cite{hammad_comprehensive_2024}. This paper will use a split phase-amplitude AF that applies $\tanh$ to the amplitude and fixes the phase \cite{hirose_continuous_1992}.

\subsection{Group Theory}
This section offers a succinct overview of the mathematics necessary to frame the problem statement and develop the methodology. 

\subsubsection{Groups and Linear Representations}
A group $G$ is an algebraic structure consisting of a set along with an associative binary operation $\cdot: G\times G \rightarrow G$. $G$ must contain an identity element $e$ such that for any $g\in G$, $g\cdot e = e\cdot g = g$. Additionally, there must exist an inverse $g^{-1}$ for every $g\in G$ such that $g^{-1}\cdot g = g \cdot g^{-1} = e$. Groups are useful structures for reasoning about symmetry.

It is often said that groups represent the symmetries of an object. However, more tooling is needed to transform an object by a symmetry encoded by a group. A group action $\triangleright$ defines how elements of a group transform some object to one of its symmetries. Formally, for a set $S$ and a group $G$, the (left) action of $G$ on $S$ is a function $\triangleright:G\times S\rightarrow S$ that satisfies the following properties \cite{grabowski_representation_2025}:
\begin{gather}
    e\triangleright s= s \quad \forall s \in S \\
    g_1 \triangleright (g_2 \triangleright s) = (g_1\cdot g_2) \triangleright s \quad\forall g_1, g_2 \in G, \ s\in S
\end{gather}

A linear representation $\rho$ is a homomorphism from a group $G$ to the group of invertible matrices over a field $\mathbb{K}$, $GL_n(\mathbb{K})$. The group action of $G$ on $\mathbb{K}^n$ is given by matrix multiplication \cite{grabowski_representation_2025}.

Of interest in this paper is the complex representation of the 2D rotation group $SO(2)$. For an angle $\theta_R$ representing the angle by which an element $R\in SO(2)$ rotates an object, the complex representation of $R$ is given by:
\begin{equation}\label{eq:repr}
    \rho_{\mathbb{C}}(R) = e^{j\theta_R}
\end{equation}
Where $j$ is the imaginary unit.

The fact that $\rho_\mathbb{C}(R)$ is a scalar in $\mathbb{C}$ makes scaling its representation to $GL_n(\mathbb{C})$ to apply 2D rotations to high-dimensional latent-space complex vectors unnecessary. Since  multiplication of a matrix or a vector by a scalar $c$ is equivalent to multiplication by the the diagonal matrix $cI_n$, the $GL_n(\mathbb{C)}$ representation of $SO(2)$ is trivial to acquire and tautological to pursue in calculations.

\subsubsection{Equivariance and Invariance}
A function $f: A\rightarrow B$ is equivariant to a group $G$ if for all $a\in A$ and $g \in G$, $f(g\triangleright_A a) = g\triangleright_B(a)$, where $\triangleright_A$ is the group action of $G$ on the domain $A$ and $\triangleright_B$ its action on the codomain $B$. $f$ is invariant to $G$ if for any $g\in G$ and $a \in A$, $f(g\triangleright_A a) = f(a)$. This is a special case of equivariance, where $\triangleright_B$ is the trivial action that maps all elements of $G$ to the identity function on the codomain, $\mathbf{id}_B$.

\section{Problem Statement}\label{prob}
This section will set up notation for a swarm of agents without a global reference frame, and present the conditions which a GNN-based controller must satisfy to remain invariant to the orientation of body frames. 


Let $\mathcal{G} =(\mathcal{V}, \mathcal{E}, \mathcal{B}, \mathbf{g}, \mathbf{e})$ be an encoded graph representing a planar swarm of robots where each robot $i\in\mathcal{V}$ has a tuple of $k$ 2D geometric features $\mathbf{g}_i = \left(g^{(i)}_1, g^{(i)}_2, \dots, g^{(i)}_k\right) \in \mathbb{R}^2$ expressed in the orthonormal basis of its body frame $\mathcal{B}_i$. Each robot occupies a point in space $\mathbf{p}_i\in\mathbb{R}^2$ in some fixed global frame, though this feature is not known to the agents. An edge $e_{ij}$ exists between two robots if the distance between them $|\mathbf{p}_i - \mathbf{p}_j|$ is less than communication radius $C$. Each edge $e_{ij}\in\mathcal{E}$ contains a 2D rotation $\mathbf{e}_{ij}\in SO(2)$ providing the rotation from frame $\mathcal{B}_i$ to $\mathcal{B}_j$. 
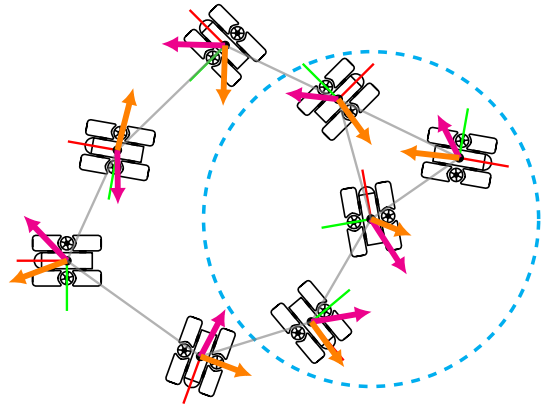
\begin{figure}
    \centering
    \vspace{0.25cm}
    \begin{adjustbox}{width=0.4\textwidth}
        \begin{tikzpicture}     
        \pic[rotate=45] (r1) at (0, 0) {frov} ;
        \pic[rotate=-100] (r2) at (2, 0.2) {frov};
        \pic[rotate=160] (r3) at (0.3, -2.3) {frov};
        \pic[rotate=-140] (r4) at (0.8, -1.5) {frov};
        \pic[rotate=80] (r5) at (-0.8, -1) {frov};
        \pic[rotate=10] (r6) at (1.2, -2) {frov};
        \pic[rotate=-45] (r7) at (0.3, 0.1) {frov};
        \pic[rotate=90] (r8) at (-1.3, -2.2) {frov};

        \draw[very thick, cyan, dashed] (r6-center) circle (2);

        \draw[link] (r6-center) -- (r4-center);
        \draw[link] (r6-center) -- (r2-center);
        \draw[link] (r6-center) -- (r7-center);

        \draw[link] (r7-center) -- (r1-center);
        \draw[link] (r7-center) -- (r2-center);

        \draw[link] (r1-center) -- (r5-center);

        \draw[link] (r5-center) -- (r8-center);

        \draw[link] (r3-center) -- (r8-center);
        \draw[link] (r3-center) -- (r4-center);

        \foreach \c in {magenta, orange} {
            \foreach \x in {1, ..., 8} {
                \pgfmathsetmacro\th{rnd*360}
                \pgfmathsetmacro\lng{rnd*0.3 + 0.5}
                \draw[enc, \c] (r\x-center) -- +(\th:\lng);
            }
        }        
    \end{tikzpicture}
    \end{adjustbox}
    
    \caption{In the swarms considered in this paper, each robot has one or several geometric features (orange and magenta arrows), such as velocity or acceleration, expressed in its body frame (red and lime green axes). Robots share a communication edge (grey line) if they are within a communication radius, indicated by the dotted cyan circle in this diagram. An edge $\mathbf{e}_{ij}$ from robot $v_i$ to robot $v_j$ comes with an encoding in $SO(2)$ representing the rotation from the body frame of robot $i$, $\mathcal{B}_i$ to that of robot $j$.}
    \label{fig:swarm}
\end{figure}



For a GNN to be invariant in the global frame to choice of local basis, it must hold that perturbing the local frames $\mathcal{B}$ does not perturb the GNN outputs as viewed in the global frame. Consider a graph $\mathcal{G}$ representing a swarm of robots and a copy $\mathcal{G}'$ where each node's local basis $\mathcal{B}_i'$ is perturbed by some rotation $\delta_i\in SO(2)$. This rotation \eqref{eq:brot} of the body frames in the global frame induces an equal and opposite rotation \eqref{eq:grot} of the geometric features in the local frame. Additionally, the rotations between frames $\mathbf{e'}$ are altered as described in \eqref{eq:erot}.
\begin{gather}
    \delta\triangleright\mathcal{G} = (\mathcal{V}, \mathcal{E}, \mathcal{B}', \mathbf{g}', \mathbf{e}') \qquad \text{where}\\
    \mathcal{B}_i' = \delta_i \triangleright \mathcal{B}_i \label{eq:brot}\\
    \mathbf{g}_i' = \left(\delta_i^{-1} \triangleright g^{(i)}_1, \delta_i^{-1} \triangleright g^{(i)}_2\dots, \delta_i^{-1}\triangleright g_k^{(i)}\right) \label{eq:grot} \\
    \mathbf{e}'_{ij} = \delta^{-1}_i\triangleright\delta_j\triangleright\mathbf{e}_{ij} \label{eq:erot}
\end{gather}
\begin{remark}[Global invariance requires local equivariance]
    Note in \eqref{eq:grot} that when $\mathcal{B}_i$ is rotated by a $\delta$, the geometric features $\mathbf{g}_i$ are be rotated an equal and opposite amount to maintain their location in the global frame. For a GNN to be globally invariant to local bases, it's output must rotate in the same manner as $\mathbf{g}_i$. Locally-speaking, this means that the GNN must \textbf{not} be invariant to frame rotation, but actually equivariant. When we speak of "basis invariance" of our proposed model, we are thinking in the global frame. However the condition in \ref{th:invar} used to enforce this invariance is a condition of equivariance to $SO(2)$ in the local frame.
\end{remark}
\begin{theorem}[Global Invariance to Local Frames]\label{th:invar}
    A GNN $\Theta$ is globally invariant to local basis if and only if for any  graph $\mathcal{G}$ and any set of frame perturbations $\delta$, it holds that:
    \[ \Theta(\delta \triangleright \mathcal{G}) = \delta \triangleright \Theta(\mathcal{G})\]
\end{theorem}

Because a GNN is a composition of aggregations, linear layers, and nonlinearities, the above criteria can be decomposed into local equivariance conditions placed on these individual components. 
\begin{lemma}\label{lm:comp}
    If each component in the whole composition that forms the GNN is independently equivariant to frame perturbations, then it must follow that the composition of all components is similarly equivariant.
\end{lemma} 
Exploiting Lemma~\ref{lm:comp}, the following section will propose a GNN layer and activation function that are equivariant to perturbations of local bases. When composed, these layers and activation functions form a GNN that is guaranteed to satisfy Theorem~\ref{th:invar}.

\section{Methods}\label{sec:methods}
This section will develop a GNN layer and activation function that are globally invariant to choice of local bases. We begin by establishing why working in the field of complex numbers simplifies the process of ensuring basis invariance. Next, we propose a complex-value SAGEConv layer that applies a pointwise rotation in the message function. Finally, we present a modification to the split phase-amplitude $\tanh$ activation function \cite{hirose_continuous_1992} that incorporates a bias on the amplitude and phase to allow for negative saturation of the activation function.

\subsection{Why Use Complex Numbers?}\label{sec:whycomp}


Working in the complex field allows us to analyze $SO(2)$ equivariance in a pointwise fashion in the latent space. Pointwise enforcement of $SO(2)$ equivariance is desirable, as vector-wise enforcement would require complicated restrictions on the properties of the weight matrix. This difficulty becomes apparent when attempting to enforce equivariance on a real-value weight matrix. Consider a real-value weight matrix $W_{\mathbb{R}}\in\mathbb{R}^{m\times 2k}$ and a vector of $k$ vertically concatenated 2D geometric features $\mathbf{x}_{\mathbb{R}_i} \in \mathbb{R}^{2k}$:
\begin{equation}\label{eq:realin}
    \mathbf{x}_{\mathbb{R}_i} = \begin{bmatrix}
        \left(g_1^{(i)}\right)^T &
        \left(g_2^{(i)}\right)^T  &
        \cdots  &
        \left(g_k^{(i)}\right)^T 
    \end{bmatrix}
\end{equation}
To prove $SO(2)$ equivariance of real-valued matrix multiplication, we would need to show that:
\begin{align}
    W_\mathbb{R}\left(\rho_a\left(\delta_i^{-1}\right)  \mathbf{x}_{\mathbb{R}_i}\right) &= \rho_b\left(\delta_i^{-1}\right) W_\mathbb{R}\mathbf{x}_{\mathbb{R}_i} \\
    \rho_b\left(\delta_i \right)W_\mathbb{R}\rho_a\left(\delta_i^{-1}\right) &= W_\mathbb{R} \label{eq:nosat}
\end{align}
Where the representation $\rho_a$ of $\delta_i^{-1}$ on $\mathbf{x}_{\mathbb{R}_i}$ is the direct sum of $k$ copies of the linear representation of $\delta_i^{-1}$ in $\mathbb{R}^2$---in other words, $\rho_a\left(\delta_i^{-1}\right)$ $k$ identical rotation matrices tiled down the diagonal. The condition \eqref{eq:nosat} does not hold for all arbitrary $W_\mathbb{R}$ regardless of choice of $\rho_b$. Significant conditions would need to be placed on $W_{\mathbb{R}}$ for \eqref{eq:nosat} to be satisfiable.

Consider now a complex formulation with a weight matrix $W_\mathbb{C}\in \mathbb{C}^{m\times k}$ and a feature vector $\mathbf{x}_{\mathbb{C}_i}\in\mathbb{C}^k$ composed of complex representations of $\mathbf{g}_i$ given by the natural isomorphism $\gamma: \mathbb{R}^2 \rightarrow \mathbb{C}$ projecting the $x$ and $y$ components of a real number onto the real and imaginary axes of the complex plane.
\begin{equation}\label{eq:compin}
    \mathbf{x}_{\mathbb{C}_i} = \left[\gamma\left(g_1^{(i)}\right), \gamma\left(g_2^{(i)}\right), \dots, \gamma\left(g_k^{(i)}\right)\right]^T
\end{equation}

Similar to the real-value formulation, multiplication by $W_\mathbb{C}$ is an $SO(2)$ equivariant operation if the following holds:
\begin{equation}
    \rho_\beta\left(\delta_i \right)W_\mathbb{C}\rho_\alpha\left(\delta_i^{-1}\right) = W_\mathbb{C} \label{eq:sat}
\end{equation}
However, unlike in \eqref{eq:nosat}, there is a simple choice of $\rho_\alpha$ and $\rho_\beta$ that satisfies \eqref{eq:sat} for any $W_\mathbb{C}$. Setting $\rho_\alpha = \rho_\beta$ to be $\rho_\mathbb{C}$---the complex scalar representation of $SO(2)$--- satisfies \eqref{eq:sat} 

\begin{theorem}[$SO(2)$ Equivariance of complex matrix multiplication]\label{th:mult}
    For any complex weight matrix $W_\mathbb{C}\in\mathbb{C}^{m\times n}$, any complex vector $\mathbf{x}\in\mathbb{C}^n$, and rotation $R\in SO(2)$ by $\theta_R$, it holds that:
    \[\rho_\mathbb{C}(R)\left(W_\mathbb{C}\mathbf{x}\right) = W_\mathbb{C}\left(\rho_\mathbb{C}(R)\mathbf{x}\right)\]
\end{theorem}
\begin{proof}
    Theorem \ref{th:mult} can be proved using the commutative property of scalar multiplication. Since $\rho_\mathbb{C}(R)$ is a scalar, it can be moved before the $W_\mathbb{C}$ on the right side of the equation, making both sides equal.
\end{proof}




\subsection{$SO(2)$-Equivariant GNN Layer}
With the motivation for using a complex-valued parameterization established, we now propose a complex-valued adaptation of the SAGEConv layer \cite{hamilton_inductive_2018}. Recall that a GNN layer in a message-passing GNN (\ref{eq:mpgnn}) is composed of a message function $\psi$, an aggregation function $\bigoplus$, and an update function $\phi$. In the vanilla SAGEConv layer, the message function is the identity function applied to the incoming message $\mathbf{x}_j$. This is appropriate when all agents (and the messages they send) share the same reference frame, however our architecture does not make this assumption. We adapt $\psi$ to rotate the incoming message into the local frame by using $\mathbf{e}_{ij}$:

\begin{equation}\label{eq:mp}
    \psi\left(\mathbf{x}_j^{(l)}, \mathbf{e}_{ij} \right) = \rho_\mathbb{C}\left(\mathbf{e}_{ij}^{-1}\right)\mathbf{x}_j^{(l)}
\end{equation}

$\psi$ satisfies the conditions for $SO(2)$ equivariance, as for any frame perturbations $\delta_i$ and $\delta_j$:
\begin{align}
    \psi\left(\mathbf{x'}_j^{(l)}, \mathbf{e'}_{ij} \right) &= \rho_\mathbb{C}\left(\mathbf\rho_\mathbb{C}\left(\delta_j\right)\rho_\mathbb{C}\left(\delta_i^{-1}\right)\mathbf{e}^{-1}_{ij}\right)\rho_\mathbb{C}\left(\delta_j^{-1}\right)\mathbf{x}_j^{(l)} \\
    &= \rho_\mathbb{C}\left(\delta_i^{-1}\right)\rho_\mathbb{C}(\mathbf{e}^{-1}_{ij})\mathbf{x}_j^{l}
\end{align}

The next step in a SAGEConv layer after message passing is aggregation. The standard SAGEConv layer uses $\text{mean}$ as the aggregation function, and our adaptation is no different. $\text{Mean}$ is a linear operator, so it is naturally equivariant to rotation. Transforming all summands in the mean by $\rho_\mathbb{C}\left(\delta_i^{-1}\right)$ has the same effect as multiplying the untransformed mean by the same factor.

The final component of a SAGEConv layer is the update function, $\phi$. In vanilla SAGEConv, this is a linear layer acting on the vertical concatenation of the node feature vector $\mathbf{x}_i^{(l)}$ and the result of aggregation. The adaptation we present uses a complex-valued weight matrix $W^{(l)}\in\mathbb{C}^{m\times k}$ with no bias, but is otherwise similar:

\begin{align}\label{eq:upd}
    \phi^{(l)}\left(\mathbf{x}_i^{(l)}, \bigoplus_{j\in \mathcal{N}(i)}\left(\mathbf{x}_j^{(l)}, \mathbf{e}_{ij}\right)\right) = \nonumber\\W^{(l)}\begin{bmatrix}
        \mathbf{x}_i^{(l)} \\
        \bigoplus_{j\in \mathcal{N}(i)}\left(\mathbf{x}_j^{(l)}, \mathbf{e}_{ij}\right)
    \end{bmatrix}
\end{align}

As established in Theorem~\ref{th:mult}, multiplication by a complex-value weight matrix is equivariant to rotation. Adding a bias in the traditional sense, however, would violate the basis-invariant properties that we seek to preserve. Since a scalar bias would remain constant in the local---not global---frame, there could be some bases where the bias annihilates an output in latent space, and others where it doubles its length. In the next subsection, we address this problem by applying a split phase-magnitude bias inside the activation function.

\subsection{$SO(2)$ Equivariant Activation Function and Bias}
The inclusion of nonlinear activation functions between layers is what gives neural networks their expressive power. However, as mentioned in section \ref{sec:cvnns}, there are many ways to apply an activation function to a complex number. Because we desire an activation function that is equivariant to rotation, we adapt the split phase-amplitude $\tanh$ function from \cite{hirose_continuous_1992}. This pointwise activation function shrinks the magnitude of each element of the latent space vector $\mathbf{c}$ to the $\tanh$ of its magnitude while leaving its direction unchanged.

We modify the split phase-amplitude $\tanh$ activation function to include a learnable real-value magnitude bias on the input, as well as a learnable rotation bias. Our method for biasing and applying an activation function on a single element of the output of the modified SAGEConv layer is formulated as follows:

\begin{equation}
    \sigma(c_i) = \begin{cases}
        \frac{e^{j\theta_i}\tanh(|c_i| + b_i) c_i}{|c_i|} & c_i\neq 0\\
        0 & c_i = 0
    \end{cases}
\end{equation}

It will now be proved that $\sigma$ satisfies the local equivariance criterion. Note that $c_i$ is an element of the output of the modified SAGEConv layer, which was previously proven equivariant to $SO(2)$, so it is \textbf{guaranteed} that rotating $\mathcal{B}_i$ by $\delta_i$ in the global frame will counter rotate $c_i$ an equal amount in the local frame. That is, $c_i' = \delta_i^{-1}c_i$. Recall that the magnitude of a complex number is unchanged by rotation; therefore (analyzing the nonzero case):
\begin{align}
    \sigma(c_i') &= \frac{e^{j\theta_i}\tanh(|\rho_\mathbb{C}(\delta^{-1})c_i| + b_i) \rho_\mathbb{C}(\delta^{-1})c_i}{|\rho_\mathbb{C}(\delta^{-1})c_i|} \\
    &= \rho_\mathbb{C}(\delta^{-1}_i)\sigma(c_i)
\end{align}

\section{Experiments and Results}
In this section, we evaluate our frame-invariant GNN on an imitation learning task for control of a swarm of robots. Similar to the experimental methodology used by \cite{tolstaya-learning-2020}, we use DAGGER imitation learning \cite{ross_reduction_2011} to learn a flocking controller proposed by \cite{tanner_stable_2003}. We compare the performance of our network to that of a network formed of standard SAGEConv layers and $\tanh$ AFs, and demonstrate that our basis-invariant model consistently outperforms this baseline. We perform a parameter sweep over the number of layers and nodes per layer to demonstrate the increased representational efficiency of the basis invariant network.
\subsection{Experimental Setup}
\subsubsection{Swarm Setup and Nominal Controller}
Our experiments consider a swarm of $n$ holonomic double integrator robots cast as an encoded graph as described in in section \ref{prob}. Each robot has a velocity $\mathbf{v}_i$ and is capable of obtaining the positions $T_{ij}$ of its neighbors in its local frame $\mathcal{B}_i$. The nominal controller uses these geometric features to calculate accelerations that align the velocities of all robots in the swarm while keeping all agents a target distance $D$ apart:
\begin{align}
    a(i) &= \sum_{j\in\mathcal{G}\setminus i} (\mathbf{v}_i - \rho(\mathbf{e}_{ij}^{-1})\mathbf{v}_j)  -2 \frac{T_{ij}}{|T_{ij}|^4} + 2\frac{T_{ij}}{|T_{ij}|^2}\\
    b(i) &= \sum_{j\in\mathcal{G}\setminus i} (\mathbf{v}_i - \rho(\mathbf{e}_{ij}^{-1})\mathbf{v}_j)\\
    u_{\text{nom}}(i) &= \begin{cases}
        a(i)  & |T_{ij}| < D \\
        a(i)& |T_{ij}| >= D\label{eq:nom}
    \end{cases}
\end{align}
As in \cite{tolstaya-learning-2020}, the nominal controller operates under the assumption of a fully-connected graph. This causes the swarm to converge uniformly and rapidly upon the steady-state behavior with minimal dispersion. The local implementation of this controller---where instead of $a(i)$ and $b(i)$ from \eqref{eq:nom} summing over all indices expect for $i$ they index over only $i$'s direct neighbors---causes the agents to splay out quickly after initialization. We use imitation learning to approximate the behavior of the globally-connected controller with only local connectivity using our basis-invariant GNN.

\subsubsection{GNN Input Features}
Because the nominal controller involves a highly non-linear relationship on the translation from a node to its neighbors, we must provide nonlinear features to the GNNs tested in our experiments. Aggregating the raw $T_{ij}$ features before applying some learned transformation and a non-linearity destroys information that the network would need to properly estimate the nominal controller. Development of a network capable of such a task is left to future work.


The network is provided the same set of non-linear features used in \cite{tolstaya-learning-2020}. The set of geometric features $\mathbf{g}_i$ used to create the input to the GNN at each node are as follows:

\begin{align}\label{eq:feat}
    a(i) &=\sum_{j\in\mathcal{N} (i)} (\mathbf{v}_i - \rho(\mathbf{e}_{ij}^{-1})\mathbf{v}_j) \nonumber\\
    \mathbf{g}_i &= \left(a(i), \ \sum_{j\in\mathcal{N}(i)} \frac{T_{ij}}{|T_{ij}|^2} , \ \sum_{j\in\mathcal{N}(i)} \frac{T_{ij}}{|T_{ij}|^4}\right)
\end{align}

\subsubsection{Network and Training Setup}
Two architectures are compared in our experiments: the basis-invariant architecture developed in section \ref{sec:methods} and a minimally-modified real-valued SAGEConv baseline model presented below. Our baseline model includes a single modification on top of vanilla SAGEConv: it appends the angle $\theta_{ij}$ corresponding to the rotation $\mathbf{e}_{ij}$ between $\mathcal{B}_i$ and $\mathcal{B}_j$ to the message from $j$ to $i$ inside the message function $\psi$. We make this modification, because it would be unfair to deprive the baseline network of edge information. Consider two 2D geometric node encodings $g_j$ and $g'_j$ of equal magnitude that lie along the $x$-axis of their local bases $\mathcal{B}_j$ and $\mathcal{B}_j'$. If we did not append the transitions $\mathbf{e}_{ij}$ and $\mathbf{e}_{ij}'$ onto the messages formed with $g_j$ and and $g_j'$ before aggregating them at node $i$, then the messages would be identical. Failing to include rotation between frames would erroneously identify messages with the same values in local frames, but different values in the global frame.

The raw feature vector \eqref{eq:realin} used in the baseline model is a stacked vector, of the 2D feature vectors defined in \eqref{eq:feat}. The raw node feature vector \eqref{eq:compin} given to the basis-invariant model is the complex representations of the features stacked.

To ensure that the networks do not learn a policy that is somehow reliant on the global frame, each agent's body frame $\mathcal{B}_i$ is initialized to a random rotation at episode start and is rotated by some normally distributed $\delta_i$ at each timestep. While the basis-invariant model is---by construction---invariant to the distribution of body frames during training, the baseline model is not. If it were trained on an environment with a global frame (i.e. all $\mathbf{e}_{ij}$ are 0 degree rotations), it would learn a policy that works only in the global reference frame environment. The baseline model needs to learn how to use angle information, so it must be exposed to a wide distribution of it during training.

Data for training is collected using the DAGGER imitation learning framework \cite{ross_reduction_2011}. Models are trained using the ADAM optimizer \cite{kingma_adam_2017} and a mean-square error loss function. The loss function for the complex-value basis invariant network is more specifically the mean-square absolute error, as regular MSE loss on complex numbers would return a complex number, and loss must be a real number.

\begin{equation}\label{eq:msae}
    \mathcal{L}(\mathbf{u}, \mathbf{u}_{\text{nom}}) = \underset{i\in\mathcal{V}}{\text{mean}} \ |u_{\text{nom}_i} - u_i|^2
\end{equation}

\subsection{Experimental Results}
\subsubsection{Representational Capacity}
\begin{figure*}
    \centering
    \includegraphics[width=1\linewidth]{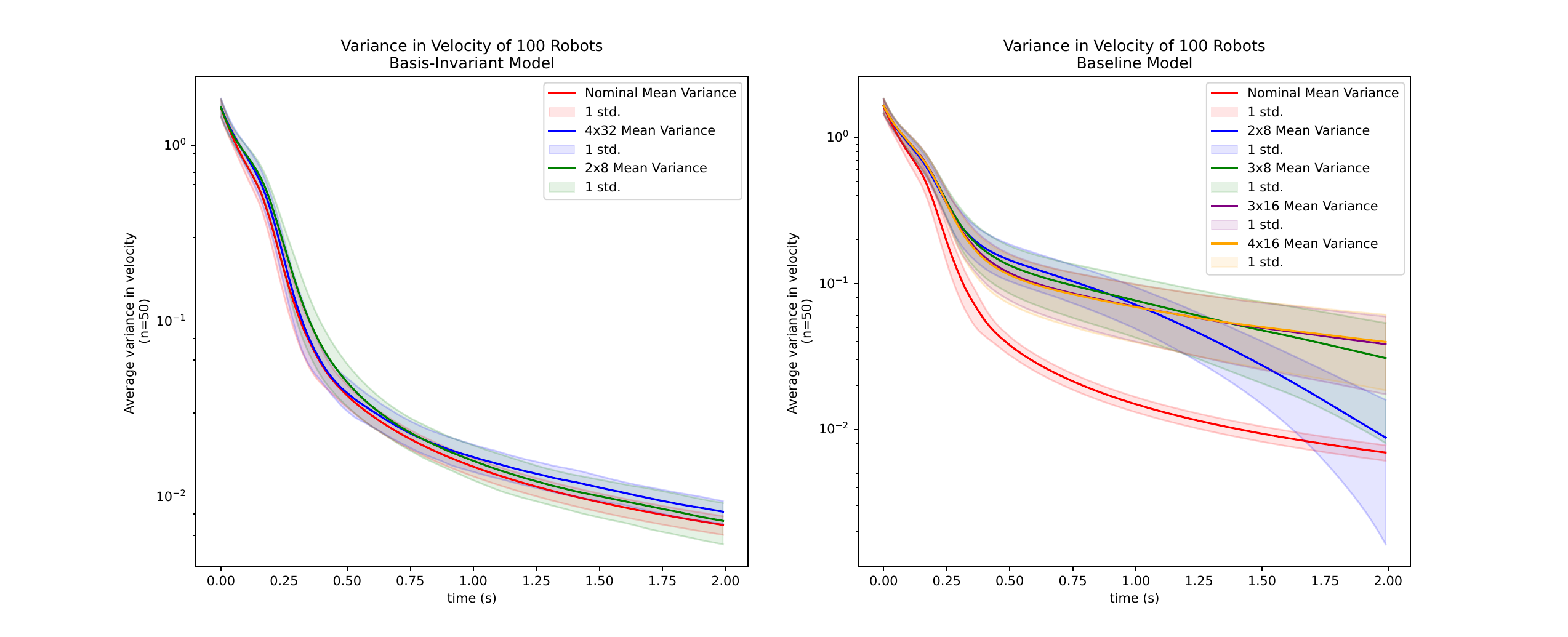}
    \caption{As seen in the left-hand plot, the smallest and largest equivariant models tested both track the nominal controller very well. The seven intermediate cases are excluded so not to crowd the plot, but they show similar results to those presented in this figure. The right-hand figure shows that the baseline GNN was not able to approximate the nominal controller as well as the basis-invariant model. Aside from not tracking the nominal controller well, the standard deviation of velocity variance is also much higher. This indicates that the controller learned by the baseline GNN is less consistent than that learned by the basis-invariant model.}
    \label{fig:mega}
\end{figure*}
To investigate how accurately and efficiently each class of GNN can learn the nominal controller, 9 versions of each are trained each with a different number of layers (2, 3, or 4) and layer width (8, 16, or 32). The variance of the velocity of all agents at each timestep is used as a metric to track the performance of each model. Models are tested on the same batch of 50 episodes.

Every single basis invariant model trained closely matched the performance of the nominal controller. As seen in the left side of Figure~\ref{fig:mega}, the velocity variance achieved by the smallest and the largest basis-invariant networks trained closely tracks the velocity variance of the nominal controller. The standard deviation of the velocity variance across trials is also quite narrow in both cases, again similar to that of the nominal controller. The fact that the basis-invariant GNN with only two 8-neuron layers is able to track the nominal controller accurately suggests that the basis-invariant architecture we proposed is efficient at learning tasks where geometric data must be converted between frames.


The performance of the baseline GNN controller does not closely match the performance of the nominal controller in any case, though it is able to learn a policy to lower velocity variance. As seen in the right side of Figure~\ref{fig:mega}, increasing the number of parameters in the baseline network actually results in a \textit{decrease} in performance. One possible reason for this performance drop is that the larger models have a higher capacity to overfit on data. A larger replay buffer and more epochs of training could recover some performance, though the baseline model's inability to generalize off the same data that the equivariant model was able to indicates that it is a less suitable candidate for learning geometric tasks on graphs without a global reference frame.

\subsubsection{Generalization}
\begin{figure*}
    \centering
    \includegraphics[width=1\linewidth]{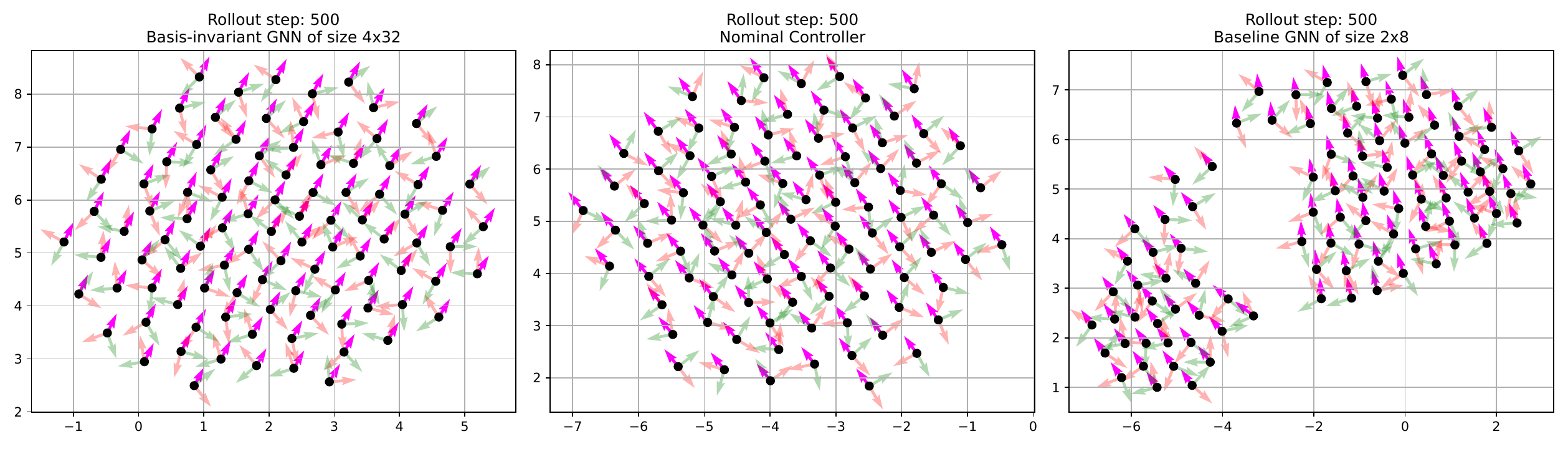}
    \caption{Shown above are the results of the extended 5 second rollouts of the basis-invariant GNN, nominal controller, and baseline GNN on an environment with 100 agents. Pink arrows indicate the velocity of each agent, and the faded red and green arrows the $x$ and $y$-axes of their body frames. The swarm visualizations for the extended-run rollouts demonstrates how similar the behavior of the basis-invariant GNN is to that of the nominal controller. The magnitudes of the velocities in the basis-invariant GNN and the nominal controller are similar to that of the baseline and the spacing is consistent. In the swarm controlled by the baseline GNN, the velocities of its agents are inconsistent and the agents are unevenly spaced. There are two contingents of agents pulling in different directions, which indicates that velocity information from far away neighbors is not being propagated or utilized well enough. }
    \label{fig:rolls}
\end{figure*}
We perform two experiments to test the ability of the model to generalize to unseen conditions: an extended run simulation, and a simulation with a reduced communication radius. 

\paragraph{Extended Run Simulation}
During training, models are only trained on episodes with a length of 2 seconds. In this experiment, we increase the episode length to 5 seconds to see how performance evolves as the models encounter increasingly unfamiliar data. The velocity variances for the best models are presented in Figure~\ref{fig:extend}, and the state of the swarms at the end of the simulation are presented in Figure~\ref{fig:rolls}.

\begin{figure*}
    \centering
    \includegraphics[width=0.9\linewidth]{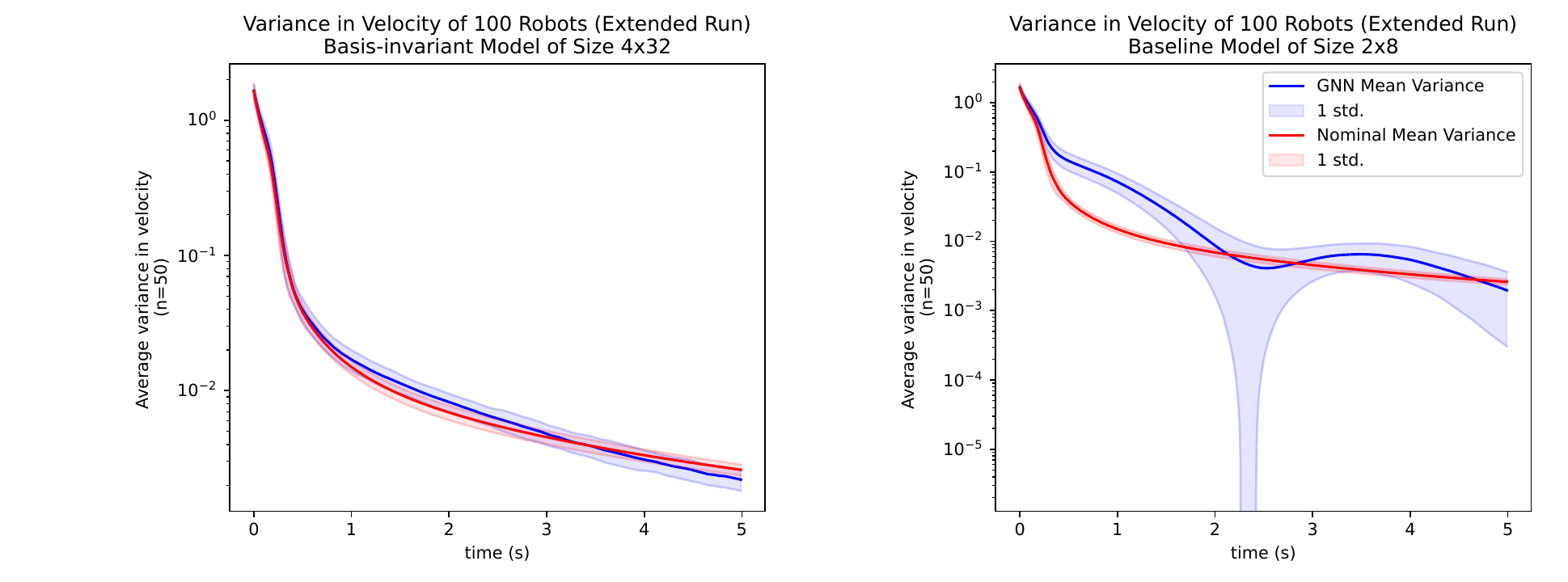}
    \caption{The extended run velocity variance of the best-performing basis-invariant and baseline models indicate different levels of generalization to situations proximal to those seen in training data. The invariant model continues to faithfully follow the performance curve of the nominal controller for the duration of the experiment. The standard deviation of the baseline model's performance increases as it encounters unfamiliar states. A spike in variance can be observed at 2.2 seconds, though the log scale does exaggerate its magnitude.}
    \label{fig:extend}
\end{figure*}

The velocity variance data in Figure~\ref{fig:extend} shows that our basis-invariant GNN approximated the nominal controller extremely well. Even in an extended run---as the swarm converges beyond a point that the GNN has ever seen---the basis-invariant model performance follows that of the nominal controller with very little deviation. Qualitatively, the shape of the basis-invariant GNN's swarm and that of the nominal controller look very similar. Figure~\ref{fig:rolls} shows that the alignment and magnitudes of agent velocities are similar, as is the pattern of agent spacing.

The baseline GNN does not completely diverge in the extended run, however Figure~\ref{fig:extend} indicates that its performance does become more unpredictable. The increase in standard deviation of velocity variance suggests that the network cannot generalize very well, and acts somewhat unexpectedly when faced with unfamiliar situations. The shape of the baseline GNN's swarm in Figure~\ref{fig:rolls} captures a larger failure of the model that is not captured with velocity variance: inaccurate inter-agent positioning. This is not unique to the extended run, though the longer runtime does allow the swarm bifurcation to progress further and take the swarm into an increasingly unfamiliar state where it generalizes poorly.

\paragraph{Reduced Communication Radius}
\begin{figure*}
    \centering
    \includegraphics[width=0.9\linewidth]{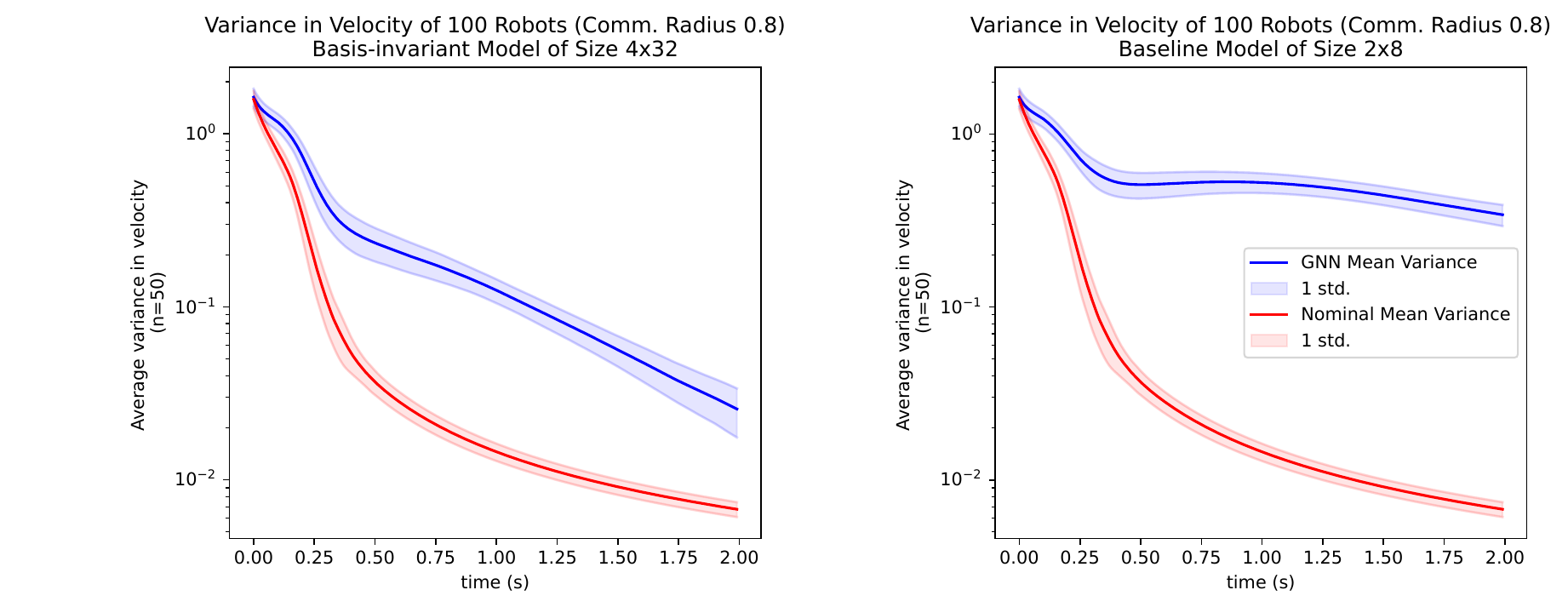}
    \caption{Lowering the communication radius increased the overall velocity variance of both the basis-invariant and baseline GNNs. However, the basis-invariant GNN is less affected, remaining within an order of magnitude of the nominal controller's velocity variance while the baseline GNN is nearly two orders of magnitude removed. }
    \label{fig:lowcomm}
\end{figure*}
To investigate how well the baseline and basis-invariant GNNs generalize to environments with different parameters, we shrink the communication radius from 1 unit to 0.8 units and measure velocity variance.

Shrinking the communication radius by just 20\% degrades the numerical performance of the basis-invariant model in terms of velocity variance, but qualitatively the controller is still able to emulate the nominal controller. As seen in Figure~\ref{fig:lowcomm}, the velocity variance of the basis-invariant model no longer tightly tracks that of the nominal controller. However, it does still reliably decrease over time. Looking at the structure of the basis-invariant GNN's swarm in Figure~\ref{fig:lowroll}, it is still very similar to that of the nominal controller. One agent did fall off the back of the swarm, and there is visibly a slight increase in the variance of the velocities' magnitudes, but agent spacing and velocity alignment is largely unaffected by the change in communication radius.

Looking now to the performance of the baseline model with a lower communication radius, a much more drastic decrease in performance is evident. Starting with the variance in velocity, the 2-layer, 8 neuron wide model whose velocity variance terminated near the nominal controller's in the standard environment (Figure~\ref{fig:mega}) gets nowhere near it in Figure~\ref{fig:lowcomm}. By the end of the simulation, velocity variance of the baseline GNN's swarm is nearly two orders of magnitude higher than that of the nominal controller. Figure~\ref{fig:lowroll} shows the directions of the agents distributed approximately uniformly in a 90 degree cone. The spacing between agent is uneven, and follows a curved path that is not present in the nominal controller.

Examining the generalization capabilities of the basis-invariant and baseline GNNs suggests that the basis-invariant model generalizes better than the baseline. While it is necessary to study the performance of this GNN architecture on more tasks before making any sweeping claims about it's broader performance characteristics, this architecture's dominance at this task suggests that it is likely well-suited for other tasks in environments without a local reference frame. The transformation of latent space messages into the local frame is a powerful inductive bias that allows the basis-invariant architecture to forego \textit{learning} how to rotate high-dimensional message vectors in the latent space, and just use them as if they were produced locally.

\begin{figure*}
    \centering
    \includegraphics[width=1\linewidth]{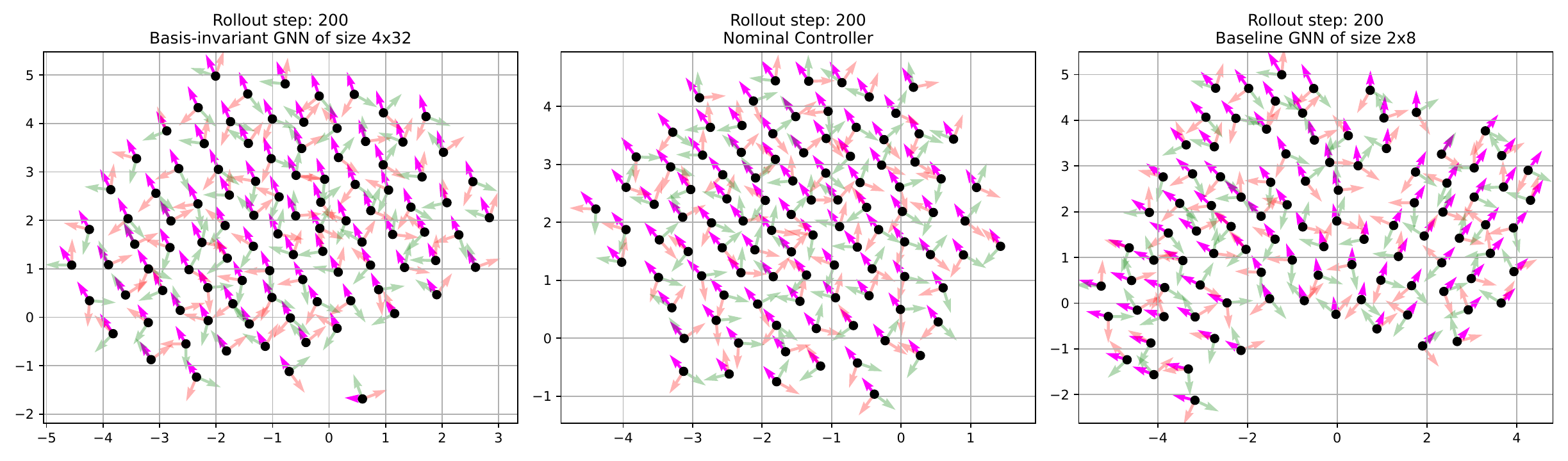}
    \caption{Reducing the communication radius had a far greater impact on the baseline GNN's performance than on the basis-invariant's. The swarm corresponding to the basis-invariant GNN is largely unchanged from Figure~\ref{fig:rolls}---the only deviations being one agent dropped off the back and a slight increase in the variance of the magnitude of the velocities. The baseline GNN's swarm on the other hand became splayed out like a fan. By the end of the 2 second simulation, there are large swaths of the swarm traveling in orthogonal directions.}
    \label{fig:lowroll}
\end{figure*}

\section{Conclusion}
This paper presents a GNN architecture for geometric control tasks on distributed systems that is analytically invariant to the choice body frames of each agent. Complex-value parameterization enables explicit transformations between latent-space vector spaces using 2D rotation information that is available in the environment. The resulting network can quickly learn tasks that use only local information expressed in local reference frames. Our model outperforms  real-valued baseline GNN in terms of representational capacity and generalization. 

One notable limitation of our proposed GNN architecture is that the model required non-linear features for it to learn our nominal controller. While the baseline GNN also needed these features, future work could design and analyze a basis-invariant GNN layer that has learnable parameters and nonlinearities before the activation function.

\bibliographystyle{IEEEtran}
\input{references.bbl}

\end{document}

%% file: rov.tex
\usepackage[utf8]{inputenc}
\usepackage{tikz}

\def\lw {0.55}
\tikzset{rov/.pic={
\begin{scope}[scale=0.05, thin]
        \path[draw=black,line cap=round,line join=round,line width=\lw] (8.5829, 
  20.9009) -- (12.4171, 20.9009) -- (12.4171, 8.7991) -- (8.5829, 8.7991) -- 
  cycle;

  \path[draw=black,line cap=round,line join=round,line width=\lw] (8.5829, 
  20.9009).. controls (8.6274, 21.3705) and (8.8515, 21.8209) .. (9.1994, 
  22.1395).. controls (9.5473, 22.4581) and (10.0156, 22.6419) .. (10.4873, 
  22.645).. controls (10.9632, 22.6482) and (11.4381, 22.4672) .. (11.7912, 
  22.1481).. controls (12.1443, 21.8289) and (12.3722, 21.3747) .. (12.4171, 
  20.9009);

  \path[draw=black,line cap=round,line join=round,line width=\lw] (15.5853, 
  23.0393).. controls (15.5853, 23.0393) and (14.364, 23.0497) .. (13.6117, 
  23.0393).. controls (12.8593, 23.0288) and (12.8504, 22.2579) .. (12.8504, 
  22.2579) -- (12.8504, 16.4983).. controls (12.8504, 16.4983) and (12.949, 
  16.0637) .. (13.2941, 16.3298).. controls (13.6393, 16.5959) and (14.1451, 
  17.0044) .. (14.6887, 17.0044)(13.8147, 23.0393).. controls (13.8147, 23.0393)
   and (15.036, 23.0497) .. (15.7883, 23.0393).. controls (16.5407, 23.0288) and
   (16.5496, 22.2579) .. (16.5496, 22.2579) -- (16.5496, 16.4983).. controls 
  (16.5496, 16.4983) and (16.451, 16.0637) .. (16.1059, 16.3298).. controls 
  (15.7607, 16.5959) and (15.2549, 17.0044) .. (14.7113, 17.0044)(15.5853, 
  6.6607).. controls (15.5853, 6.6607) and (14.364, 6.6503) .. (13.6117, 
  6.6607).. controls (12.8593, 6.6712) and (12.8504, 7.4421) .. (12.8504, 
  7.4421) -- (12.8504, 13.2017).. controls (12.8504, 13.2017) and (12.949, 
  13.6363) .. (13.2941, 13.3702).. controls (13.6393, 13.1041) and (14.1451, 
  12.6956) .. (14.6887, 12.6956)(13.8147, 6.6607).. controls (13.8147, 6.6607) 
  and (15.036, 6.6503) .. (15.7883, 6.6607).. controls (16.5407, 6.6712) and 
  (16.5496, 7.4421) .. (16.5496, 7.4421) -- (16.5496, 13.2017).. controls 
  (16.5496, 13.2017) and (16.451, 13.6363) .. (16.1059, 13.3702).. controls 
  (15.7607, 13.1041) and (15.2549, 12.6956) .. (14.7113, 12.6956)(5.4147, 
  23.0393).. controls (5.4147, 23.0393) and (6.636, 23.0497) .. (7.3883, 
  23.0393).. controls (8.1407, 23.0288) and (8.1496, 22.2579) .. (8.1496, 
  22.2579) -- (8.1496, 16.4983).. controls (8.1496, 16.4983) and (8.051, 
  16.0637) .. (7.7059, 16.3298).. controls (7.3607, 16.5959) and (6.8549, 
  17.0044) .. (6.3113, 17.0044)(7.1853, 23.0393).. controls (7.1853, 23.0393) 
  and (5.964, 23.0497) .. (5.2117, 23.0393).. controls (4.4593, 23.0288) and 
  (4.4504, 22.2579) .. (4.4504, 22.2579) -- (4.4504, 16.4983).. controls 
  (4.4504, 16.4983) and (4.549, 16.0637) .. (4.8941, 16.3298).. controls 
  (5.2393, 16.5959) and (5.7451, 17.0044) .. (6.2887, 17.0044)(5.4147, 6.6607)..
   controls (5.4147, 6.6607) and (6.636, 6.6503) .. (7.3883, 6.6607).. controls 
  (8.1407, 6.6712) and (8.1496, 7.4421) .. (8.1496, 7.4421) -- (8.1496, 
  13.2017).. controls (8.1496, 13.2017) and (8.051, 13.6363) .. (7.7059, 
  13.3702).. controls (7.3607, 13.1041) and (6.8549, 12.6956) .. (6.3113, 
  12.6956)(7.1853, 6.6607).. controls (7.1853, 6.6607) and (5.964, 6.6503) .. 
  (5.2117, 6.6607).. controls (4.4593, 6.6712) and (4.4504, 7.4421) .. (4.4504, 
  7.4421) -- (4.4504, 13.2017).. controls (4.4504, 13.2017) and (4.549, 13.6363)
   .. (4.8941, 13.3702).. controls (5.2393, 13.1041) and (5.7451, 12.6956) .. 
  (6.2887, 12.6956);

  \path[draw=black,line cap=round,line join=round,line 
  width=\lw,shift={(-0.0529, -0.01)}] (16.2732, 14.86)arc(0.0:90.0:1.5732 
  and -1.5405)arc(90.0:180.0:1.5732 and -1.5405)arc(180.0:270.0:1.5732 and 
  -1.5405)arc(-90.0:0.0:1.5732 and -1.5405) -- cycle;

  \path[draw=black,line cap=round,line join=round,line width=\lw,cm={ 
  -1.0,-0.0,-0.0,1.0,(21.0529, -0.01)}] (16.2732, 14.86)arc(0.0:90.0:1.5732 and 
  -1.5405)arc(90.0:180.0:1.5732 and -1.5405)arc(180.0:270.0:1.5732 and 
  -1.5405)arc(-90.0:0.0:1.5732 and -1.5405) -- cycle;

  \path[draw=black,fill=black,line cap=round,line join=round,line width=\lw] 
  (15.1495, 14.85).. controls (15.1495, 14.5725) and (14.9246, 14.3476) .. 
  (14.6471, 14.3476).. controls (14.3696, 14.3476) and (14.1446, 14.5725) .. 
  (14.1446, 14.85).. controls (14.1446, 15.1275) and (14.3696, 15.3524) .. 
  (14.6471, 15.3524).. controls (14.9246, 15.3524) and (15.1495, 15.1275) .. 
  (15.1495, 14.85) -- cycle(5.8505, 14.85).. controls (5.8505, 14.5725) and 
  (6.0754, 14.3476) .. (6.3529, 14.3476).. controls (6.6304, 14.3476) and 
  (6.8554, 14.5725) .. (6.8554, 14.85).. controls (6.8554, 15.1275) and (6.6304,
   15.3524) .. (6.3529, 15.3524).. controls (6.0754, 15.3524) and (5.8505, 
  15.1275) .. (5.8505, 14.85) -- cycle;

  \path[draw=black,fill=black,line cap=round,line join=round,line width=\lw] 
  (14.6471, 14.85) -- (16.1163, 15.1999)(14.6471, 14.85) -- (15.0786, 
  16.2973)(14.6471, 14.85) -- (13.6095, 15.9474)(14.6471, 14.85) -- (13.1779, 
  14.5001)(14.6471, 14.85) -- (14.2155, 13.4027)(14.6471, 14.85) -- (15.6847, 
  13.7526)(6.3529, 14.85) -- (4.8837, 15.1999)(6.3529, 14.85) -- (5.9214, 
  16.2973)(6.3529, 14.85) -- (7.3905, 15.9474)(6.3529, 14.85) -- (7.8221, 
  14.5001)(6.3529, 14.85) -- (6.7845, 13.4027)(6.3529, 14.85) -- (5.3153, 
  13.7526);

  \path[draw=black,fill=black,line cap=round,line join=round,line width=\lw] 
  (8.1496, 19.3781) -- (8.5829, 19.3781) -- (8.5829, 17.9024) -- (8.1496, 
  17.9024) -- cycle(12.8504, 19.3781) -- (12.4171, 19.3781) -- (12.4171, 
  17.9024) -- (12.8504, 17.9024) -- cycle(8.1496, 10.3219) -- (8.5829, 10.3219) 
  -- (8.5829, 11.7976) -- (8.1496, 11.7976) -- cycle(12.8504, 10.3219) -- 
  (12.4171, 10.3219) -- (12.4171, 11.7976) -- (12.8504, 11.7976) -- cycle;  
  \coordinate (-center) at (10.5, 14.85);
  \coordinate (-rt) at (14.6471, 14.85);
  \coordinate (-lt) at (6.3529, 14.85);
  \coordinate (-bow) at (10.5, 22.645);
\end{scope}
    
    }
}